\documentclass{article}

\PassOptionsToPackage{numbers}{natbib}
\PassOptionsToPackage{table}{xcolor}

    \usepackage[preprint]{neurips_2025}

\usepackage[utf8]{inputenc} 
\usepackage[T1]{fontenc}    
\usepackage{hyperref}       
\usepackage{url}            
\usepackage{booktabs}       
\usepackage{amsfonts}       
\usepackage{nicefrac}       
\usepackage{microtype}      
\usepackage{xcolor}         
\usepackage{amsthm}
\usepackage{amsmath} 
\usepackage{amssymb}
\usepackage{empheq}
\usepackage{mathtools}
\usepackage{cleveref}
\usepackage{listings}
\usepackage{makecell}
\usepackage{thmtools}
\usepackage{thm-restate}

\usepackage{sidecap}
\usepackage[most]{tcolorbox}
\newtcolorbox{conclusionbox}{
  colback=blue!5,        %
  colframe=blue!75!black, %
  coltitle=black,        %
  fonttitle=\bfseries,   %
  boxrule=1pt,           %
  arc=1mm,               %
  left=2mm,              %
  right=2mm,             %
  top=1mm,               %
  bottom=1mm,            %
}
\usepackage{algorithm}
\usepackage{algorithmic}

\usepackage{tikz}       
\usepackage{pgfmath}    

\title{Next Token Perception Score: Analytical Assessment of your LLM Perception Skills}

\author{%
  Yu-Ang Cheng\thanks{These authors contributed equally to this work.} \\
  Department of Cognitive \& Psychological Sciences\\
  Brown University\\
  \texttt{yuang\_cheng@brown.edu} \\
  \And
  Leyang Hu\footnotemark[1] \\
  Department of Computer Science \\
  Brown University \\
  \texttt{leyang\_hu@brown.edu} \\
  \And
  Hai Huang \\
  Atlassian \\
  \texttt{hhuang3@atlassian.com} \\
  \And
  Randall Balestriero \\
  Department of Computer Science \\
  Brown University \\
  \texttt{randall\_balestriero@brown.edu} \\
}

\begin{document}

\maketitle

\begin{abstract}
    Autoregressive pretraining has become the de facto paradigm for learning general-purpose representations in large language models (LLMs). However, linear probe performance across downstream perception tasks shows substantial variability, suggesting that features optimized for next-token prediction do not consistently transfer well to downstream perception tasks. We demonstrate that representations learned via autoregression capture features that may lie outside the subspaces most informative for perception. To quantify the (mis)alignment between autoregressive pretraining and downstream perception, we introduce the Next Token Perception Score (NTPS)—a score derived under a linear setting that measures the overlap between autoregressive and perception feature subspaces. This metric can be easily computed in closed form from pretrained representations and labeled data, and is proven to both upper- and lower-bound the excess loss. Empirically, we show that NTPS correlates strongly with linear probe accuracy across 12 diverse NLP datasets and eight pretrained models ranging from 270M to 8B parameters, confirming its utility as a measure of alignment. Furthermore, we show that NTPS increases following low-rank adaptation (LoRA) fine-tuning, especially in large models, suggesting that LoRA aligning representations to perception tasks enhances subspace overlap and thus improves downstream performance. More importantly, we find that NTPS reliably predicts the additional accuracy gains attained by LoRA finetuning thereby providing a lightweight prescreening tool for LoRA adaptation. Our results offer both theoretical insights and practical tools for analytically assessing LLM perception skills.
\end{abstract}

\section{Introduction}
\label{sec:introduction}
\begin{SCfigure}[][h]
    \centering
    \includegraphics[scale=0.75]{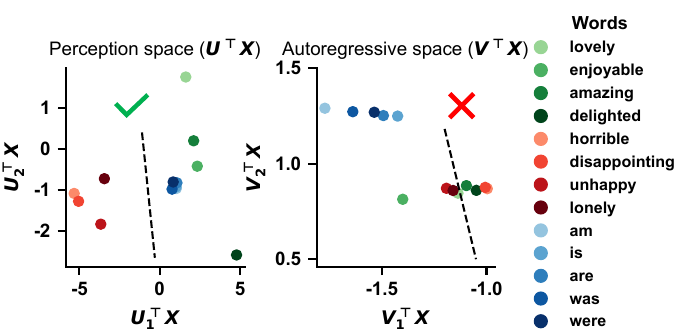}
      \caption{2D projections of the representations learned through perception \(U^\top X\) (left) or autoregressive \(V^\top X\) (right) tasks based on OpenELM-450M and Emotion\citep{emotion} dataset. \textbf{In the perception space, joy and sad words are clearly separated, illustrating label-relevant semantic encoding; in the autoregressive space, these two emotion classes overlap, but syntactic categories become clearly clustered.}}    
    \label{fig:uv-projections}
\end{SCfigure}
The success of GPT-1 \citep{gpt1} demonstrated the effectiveness of autoregressive pretraining—training a model to predict the next token given preceding context—for learning transferable language representations. This autoregressive training paradigm quickly became the standard for building large language models (LLMs), leading to increasingly capable successors such as GPT-2/3/4 \citep{gpt2, gpt3, gpt4} and LLaMA-1/2/3/4 \citep{llama1, llama2, llama3, llama4}. As model capabilities scaled, the expectation evolved: rather than merely serving as a source of transferable features requiring task-specific fine-tuning, a foundation model is now expected to perform well on downstream perception tasks straight out of the box, without modifying its pretrained weights.

Under this expectation, prompting has emerged as a popular strategy \citep{gpt2}—adapting the model to downstream perception tasks solely by manually crafting input text. Meanwhile, linear probing has gained traction as an alternative \citep{kumar2022fine, liu2023same}, leveraging frozen representations and training only a lightweight linear classifier on top, offering a more structured and efficient approach to downstream adaptation. While prompting and linear probing may appear quite different in practice, they are fundamentally two sides of the same coin. In prompting, natural language input steers the model, and the model’s pretrained linear head for autoregression maps hidden states to output token probabilities. In linear probing, by contrast, the model is frozen and an external linear head is trained to interpret its hidden representations.

However, as we show in \Cref{sec:alignment_score_computation}, linear probing does not perform equally well across all downstream perception tasks. In some cases, training a model from scratch on the downstream dataset yields significantly better performance. This suggests that representations learned by current LLMs are not universally effective: while some downstream tasks benefit greatly from pretraining and are well-served by simple adaptation methods like linear probing, others are not. Motivated by this observation, we pose the following question:

\textbf{\emph{How can we quantify the alignment between autoregressive pretraining and downstream perception tasks to explain the varying effectiveness of linear probing across different datasets?}}

In this paper, we take a first step towards understanding and assessing the alignment between autoregressive pretraining and downstream perception tasks by:

\begin{itemize}
\item \textbf{Systematically evaluating the benefits of autoregressive pretraining}—by comparing linear‐probe performance on six pretrained models against identical architectures trained from scratch across 12 downstream perception datasets.

\item \textbf{Proposing the Next Token Perception Score (NTPS)}—a metric that quantifies the alignment between autoregressive pretraining and downstream perception tasks by measuring the overlap between their respective feature subspaces.

\item \textbf{Empirically validating the reliability of NTPS}—by demonstrating that it not only correlates strongly with linear probe accuracy across 12 diverse datasets and eight pretrained models, but also consistently increases after task-specific LoRA fine-tuning.

\item \textbf{Guiding LoRA finetuning with NTPS predictions}—by demonstrating that NTPS reliably forecasts the additional accuracy gains from LoRA finetuning, thereby providing a lightweight prescreening tool.

\end{itemize}

The remainder of the paper is organized as follows. \Cref{sec:background} reviews both empirical and theoretical foundations for applying pretrained LLMs directly to downstream tasks. \Cref{sec:alignment_score_computation} presents evidence that pretrained LLM representations are not universally effective when used directly through linear probing and therefore we introduce our proposed NTPS metric to quantify such (mis)alignment. \Cref{sec:alignment_performance_correlation} empirically demonstrates that NTPS correlates with linear probing performance across 12 downstream datasets and eight pretrained models and shows that NTPS consistently increases after task-specific LoRA fine-tuning. \Cref{sec:alignment_performance_correlation} further suggests that NTPS can serve as a good predictor of performance gain from LoRA finetuning. \Cref{sec:conclusion} summarizes our findings. Our code is available at: \url{https://github.com/Yu-AngCheng/NTPS}.

\section{Background}
\label{sec:background}
\paragraph{Utilizing Pretrained LLM Hidden Representations for Downstream Tasks}
Pretrained large language models can be leveraged for downstream tasks without any gradient‐based fine-tuning via two complementary strategies: prompt-based in-context learning and linear probing of fixed hidden representations. In prompting, GPT-3 attains strong zero- and few-shot classification performance across diverse NLP benchmarks using only natural-language templates and without any weight updates \citep{gpt3}. Likewise, when given only in-context examples, GPT-3 can perform both linear and non-linear regression at levels comparable to Random Forests and Gradient Boosting \citep{vacareanu2024words}. In linear probing, early work reveals that deep linguistic structures are capturable by training simple classifiers on fixed frozen representations \citep{tenney2019you,jawahar2019does}. Recent work show that LLM embeddings preserve Lipschitz continuity and can be utilized in high-dimensional regression settings, outperforming conventional feature-engineering baselines \citep{tang2024understandingLLM}.

\paragraph{Scaling laws for Predicting Downstream Performance in LLMs}
The downstream performance of LLMs has garnered significant attention, with special focus on scaling laws. \citet{gadre2024language} have found a power law between validation loss and model FLOPs, and also a power law between language model perplexity and average performance across all datasets. \citet{isik2025scaling} have found a log law in translation tasks between downstream translation performance and the number of tokens in the pretrained task. \citet{chen2024scaling} have proposed a two-stage approach to predict downstream performance: first mapping computational resources to pretraining loss based on a power law, then mapping pretraining loss to downstream task performance based on a linear mapping. Although these formulas achieve reasonable forecasts, they still rely on finetuning smaller models for calibration and offer limited mechanistic insight into why certain tasks benefit more from scale.

\paragraph{Theoretical Foundations of Utilizing Pretrained LLM Hidden Representations}
Recent studies in understanding pretraining objectives have revealed precise conditions under which different self-supervised losses guarantee—or fail to guarantee—strong downstream performance. \citet{balestriero2024learning} rigorously demonstrates that reconstruction‐based training, such as autoencoders or masked reconstruction, can produce features that capture all input variance yet remain uninformative for discriminative perception, underscoring that low reconstruction error alone is insufficient for transfer.   \citet{wu2023connecting} identify two necessary conditions for autoregressive next-token models to transfer effectively: the softmax head must break shift invariance, and downstream tasks must not hinge on tokens given vanishingly small pretraining probabilities. \citet{liu2023same} show that among models with identical pretraining loss, those converging to flatter minima generalize best—revealing that the implicit bias of optimization plays a crucial role in shaping downstream performance.



\section{Next Token Perception Score (NTPS): an analytical assessment metric of pretrained LLMs}
\label{sec:alignment_score_computation}
In this section, we first present empirical evidence in \Cref{subsec:evidence} showing that, while pretrained LLM representations can boost performance on some perception tasks, they may underperform or offer no benefit on others when compared to models trained from scratch. We then develop a linear‐regime theory in \Cref{subsec:linear theory} that characterizes the optimal feature maps for next‐token prediction versus downstream regression, and formally relates their alignment to excess task loss. Finally, in \Cref{divergence_illus} we illustrate how autoregressive (V) and perception‐trained (U) subspaces diverge based on our theory.

\subsection{On the Need to Monitor LLM Alignment for Perception Task}
\label{subsec:evidence}
To demonstrate that LLM representations are not universally effective, we compare the linear probing performance of pretrained models on downstream perception tasks with the performance of the same architectures trained from scratch on the downstream perception datasets.

Here, we evaluate six models: Qwen2-0.5B/1.5B \citep{qwen2} and OpenELM-270M/450M/1.1B/3B \citep{openelm}. The evaluation spans 12 downstream datasets across a variety of domains, including: Intent Classification \citep{intent_classification}, Clickbait Title Classification \citep{clickbait}, SST-2 \citep{sst2}, Bias Identification \citep{bias_identification}, Banking\citep{bank77}, Emotion\citep{emotion}, SMS Spam \citep{sms_spam}, Medical Question Pairs \citep{medical}, Rotten Tomatoes \citep{rotten_tomatoes}, CommonsenseQA \citep{commonsenseqa}, Climate Sentiment \citep{climate_sentiment}, and IMDB \citep{imdb}.

For full training, we use the following configuration across all models and datasets: Adafactor optimizer with a learning rate of $10^{-4}$, $\epsilon$ values of $10^{-30}$ and $10^{-3}$, gradient clipping threshold of 1.0, decay rate of 0.8, and weight decay of $10^{-5}$; 10, 000 training steps with a cosine learning rate scheduler with a 5\% warm-up phase. For linear probing: we use the following configuration across all models and datasets: AdamW optimizer with a learning rate of $10^{-4}$; 50 epochs. For both cases, we extract the mean token representation from the final transformer block and feed it into a linear classification head.

As shown in \cref{tab:evidence}\footnote{The full training record for OpenELM 3B is unavailable due to insufficient GPU memory (A100), even when using a batch size of 1.}, the effect of autoregressive pretraining with linear probing varies markedly across datasets. On sentiment‐analysis tasks—SST-2 \citep{sst2}, Rotten Tomatoes \citep{rotten_tomatoes}, Climate \citep{climate_sentiment} and IMDB \citep{imdb}—linear probing delivers gains of roughly 5–10\%. For intent classification \citep{intent_classification}, clickbait detection \citep{clickbait}, bias identification \citep{bias_identification}, SMS spam \citep{sms_spam} and CommonsenseQA \citep{commonsenseqa}, the performance difference between linear probing and training from scratch stays within about 1\%. In the most extreme cases—emotion recognition \citep{emotion} and medical‐text classification \citep{medical}—linear probing actually underperforms training from scratch by a substantial margin. 

\begin{table}[t]
  \centering
  \caption{Comparison of linear probing performance of pretrained models versus full-training from scratch across downstream datasets. \textbf{Linear probing can outperform, match, or underperform full-training from scratch, indicating that pretrained LLM representations are not universally effective.}}
  \setlength{\tabcolsep}{2pt}
  \begin{tabular}{l*{12}{c}}
    \toprule
    & \multicolumn{2}{c}{\makecell{Qwen2\\0.5B}} 
    & \multicolumn{2}{c}{\makecell{Qwen2\\1.5B}}
    & \multicolumn{2}{c}{\makecell{OpenELM\\270M}} 
    & \multicolumn{2}{c}{\makecell{OpenELM\\450M}}
    & \multicolumn{2}{c}{\makecell{OpenELM\\1.1B}} 
    & \multicolumn{2}{c}{\makecell{OpenELM\\3B}} \\
    \cmidrule(r){2-3} \cmidrule(r){4-5} \cmidrule(r){6-7} 
    \cmidrule(r){8-9} \cmidrule(r){10-11} \cmidrule(r){12-13}
    & Linear & Full & Linear & Full & Linear & Full & Linear & Full & Linear & Full & Linear & Full \\
    \midrule
    Intent           & 99.7 & 99.6 & 99.9 & 99.5 & 99.3 & \textbf{99.6} & 99.5 & \textbf{99.6} & 99.8 & 99.8 & 98.4 & \textbf{99.0} \\
    Clickbait Title  & 99.4 & 99.1 & 99.6 & 99.0 & 99.4 & 98.4 & 99.6 & 98.4 & 99.7 & 98.7 & 99.6 & 98.6 \\
    SST-2            & 85.4 & 80.4 & 88.2 & 82.1 & 87.6 & 80.3 & 87.7 & 82.5 & 89.3 & \textbf{92.0} & 89.9 & 78.7  \\
    Banking          & 88.1 & 85.4 & 89.4 & 82.4 & 89.8 & 86.3 & 90.5 & 84.8 & 91.3 & 83.3 & 82.0 & \textbf{82.3}  \\
    Bias             & 95.5 & 94.9 & 96.4 & 94.6 & 96.5 & 94.7 & 96.4 & 94.4 & 96.8 & 95.5 & 95.4 & 94.8  \\
    Emotion          & 66.2 & \textbf{88.3} & 69.0 & \textbf{88.3} & 70.9 & \textbf{86.8} & 72.0 & \textbf{86.7} & 73.6 & \textbf{76.9} & 63.6 & \textbf{87.9}  \\
    SMS Spam         & 99.3 & 98.7 & 99.3 & 98.9 & 99.0 & 98.8 & 99.2 & 98.7 & 99.2 & 98.9 & 98.4 & \textbf{99.0}  \\
    Medical          & 36.4 & \textbf{51.5} & 28.9 & \textbf{51.3} & 33.8 & \textbf{51.5} & 30.6 & \textbf{51.5} & 27.5 & \textbf{51.5} & 36.7 & \textbf{51.5}  \\
    Rotten Tomatoes  & 81.9 & 75.9 & 85.6 & 73.5 & 82.6 & 74.1 & 84.1 & 76.9 & 86.8 & 75.2 & 84.8 & 74.9  \\
    Commonsense      & 22.1 & 21.0 & 24.2 & 22.4 & 22.3 & 21.2 & 21.2 & \textbf{22.2} & 23.3 & 21.3 & 21.5 & 21.3  \\
    Climate          & 78.4 & 63.7 & 81.3 & 67.8 & 80.3 & 69.1 & 79.1 & 71.9 & 81.6 & 69.1 & 79.4 & 71.2  \\
    IMDB             & 92.5 & 86.0 & 94.4 & 84.9 & 92.8 & 84.2 & 99.5 & 84.1 & 94.5 & 83.5 & 94.4 & –    \\
    \bottomrule
  \end{tabular}
  \label{tab:evidence}
\end{table}

This variability suggests that the representations learned via autoregressive pretraining do not uniformly align with downstream perception tasks. Therefore, we are going to quantify this alignment (or lack thereof). As a starting point, we first build intuition and establish theoretical results in the \emph{linear regime}. As we will show in \cref{sec:alignment_performance_correlation}, this seemingly simplified setting provides surprisingly informative insights into more complex empirical scenarios.

\begin{conclusionbox}
\textbf{Takeaway}: Linear probing on pretrained LLM representations can outperform, match, or underperform full-training from scratch.
\end{conclusionbox}

\subsection{Theoretical foundations: linear approximations}
\label{subsec:linear theory}
Consider a sentence, whose representation is $X \in \mathbb{R}^{(d, \ell)}$, where $d$ is the hidden dimension size of each token and $\ell$ is the total number of tokens in this sentence. Consider two variants $X_1, X_2 \in \mathbb{R}^{d \times (\ell-1)}$ from $X$, where $i$-th column of $X_1$ is the representation of the first $i$ tokens of the sentence and $i$-th column of $X_2$ is the representation of the $i+1$-th token of the sentence.

Autoregressive training aims to find a model's parameter $\theta$ to predict $X_2$ based on $X_1$. Specifically, the training objective is to minimize the following Cross-Entropy(CE) loss:
\begin{equation}
\mathcal{L}_{\mathrm{CCE}}
= -\mathbb{E}_{X}\bigl[\log p_\theta(X_2 \mid X_1)\bigr],
\qquad
p_\theta(X_2 \mid X_1)
\propto g_\theta(f_\theta (X_1)).
\end{equation}

In the linear setting, \(f_\theta\) is a linear map \(V \in \mathbb{R}^{d \times k}\) and \(g_\theta\) is another linear map \(W \in \mathbb{R}^{k \times d}\). Besides, we assume the $i$-th column of $X_1$ represents the sum of the first $i$ tokens in $X$ and the $i$-th column of $X_2$ represents the $i+1$-th token in $X$. 

Then the loss function $\mathcal{L}$ is defined as:
\begin{equation}\label{eq:autoregressive}
\mathcal{L}
= \mathbb{E}_{X_1,X_2}\bigl\lVert W^\top V^\top X_1 - X_2 \bigr\rVert_F^2  = \mathbb{E}_{X}\bigl\lVert W^\top V^\top X\,L_1 - X\,L_2 \bigr\rVert_F^2.
\end{equation}
Here $L_1, L_2 \in \mathbb{R}^{\ell \times (\ell-1)}$ is for selecting the tokens in $X$, see \Cref{app:thm1} for full definition.

Similarly, given another pair ($U \in \mathbb{R}^{d \times k}$, $Z \in \mathbb{R}^{k \times c}$). When we use the sum of all tokens in the sentence to predict the label \(Y \in \mathbb{R}^{c}\), our loss $\mathcal{L}^*$ is defined below:
\begin{equation}\label{eq:perception}
    \mathcal{L^*} = \mathbb{E}_{X, Y} \left\| Z^\top U^\top X \, 1_{\mathrm{\ell \times 1}} - Y \right\|_F^2,
\end{equation}
Note that in both cases, instead of the CE loss, we assume a convex mean square error (MSE) loss. We now state a key guarantee for our choice: as the MSE loss vanishes, so does the probability of decoding error.  

\begin{restatable}[Equivalence between MSE and CE]{lemma}{EquivMSECE}
\label{lem:equiv-mse-ce}
Let \(X\in\mathbb R^{d\times \ell}\) be the token representations
Denote
\[
  h^* = X\,L_2,\qquad
  \hat h = W^\top V^\top X\,L_1.
\]
Assume the vocabulary embeddings \(\{w_i\}_{i=1}^V\subset\mathbb R^d\) satisfy a positive margin
\begin{equation}
  \Delta \;=\;\min_{j\neq y}\bigl\langle w_y, h^*\bigr\rangle
  - \bigl\langle w_j, h^*\bigr\rangle \;>\;0,
  \quad
  M \;=\;\max_{i\neq j}\|w_i - w_j\|_2.
\end{equation}
If $ \mathcal L \;=\;\mathbb{E} \|\hat h - h^*\|_F^2\;\to\;0,$
then
\begin{equation}
  \Pr\bigl(\arg\max_i\langle w_i,\hat h\rangle = y\bigr)
  \;\longrightarrow\;1.
\end{equation}
\end{restatable}
\textit{The proof of \Cref{lem:equiv-mse-ce} is deferred to \Cref{app:lemma1} (for empirical validation, see \citep{hui2020evaluation})}.

Now we can solve the \cref{eq:autoregressive} and \cref{eq:perception} under the following theorem.
\begin{restatable}{theorem}{ModelOptTheorem}
\label{thm:1}
The loss functions \(\mathcal{L}\) in \cref{eq:autoregressive} and \(\mathcal{L}^*\) in \cref{eq:perception} are minimized for
\begin{equation}
    W = (V^\top \mathbb{E}[X\,L_1 L_1^\top \, X ]V)^{-1}V^\top \mathbb{E}[X\,L_1 L_2^\top \, X ] 
\end{equation}
\begin{equation}
    Z = (U^\top \mathbb{E}[X\,1_{\mathrm{\ell \times 1}} 1_{\mathrm{\ell \times 1}}^\top \, X ]U)^{-1}U^\top \mathbb{E}[X\,1_{\mathrm{\ell \times 1}} Y^\top]
\end{equation}
\(U, V\) span the top \(k\) eigenvectors of the following generalized eigenvalue problems:
\begin{empheq}[box=\fbox]{align}
  \mathbb{E}[X\,L_1 L_2^\top X^\top]\,
    \mathbb{E}[X\,L_2 L_1^\top X^\top]\,\tilde{V}
    &= \mathbb{E}[X\,L_1 L_1^\top X^\top]\,\tilde{V}\,\Lambda_{V},\\
  \mathbb{E}[X\,1_{\ell\times1}\,Y^\top]\,
    \mathbb{E}[Y\,1_{1\times\ell}\,X^\top]\,\tilde{U}
    &= \mathbb{E}[X\,1_{\ell\times\ell}\,X^\top]\,\tilde{U}\,\Lambda_{U}.
\end{empheq}
\end{restatable}
\textit{The proof of \Cref{thm:1} is deferred to \Cref{app:thm1}}.

From Theorem~\ref{thm:1}, it shows that $U$ and $V$ capture distinct co-variability structures; hence, the autoregressively derived $V$ may not generalize well to downstream tasks that depend on $U$. 

\begin{conclusionbox}
\textbf{Takeaway}: Under a linear regime, autoregressive and perceptual objectives correspond to distinct generalized eigenvalue problems and thus learn feature that capture different covariance structures. 
\end{conclusionbox}

\subsection{Divergence of autoregressive ($V$) vs. perceptual ($U$) features}
\label{divergence_illus}

\begin{figure}[hbtp]
\centering
\begin{minipage}{\linewidth}
\begin{algorithm}[H]
\caption{Computation of NTPS}
\label{alg:ntps}
\begin{algorithmic}[1]
\REQUIRE Dataset $\mathcal{D} = \{(x_i, y_i)\}_{i=1}^n$, pretrained transformer $f$, tokenizer $\mathcal{T}$, hidden dimension $d$, subspace dimension $k$, target layer $l$
\ENSURE NTPS at layer $l$  using top-$k$ subspace

\STATE Initialize: 
\texttt{meanXX}$ \in \mathbb{R}^{d \times d}$, 
\texttt{meanXY}$ \in \mathbb{R}^{d \times c}$, 
\texttt{cov0}$ \in \mathbb{R}^{d \times d}$, 
\texttt{cov1}$ \in \mathbb{R}^{d \times d}$

\FOR{each sample $(x, y) \in \mathcal{D}$}
    \STATE Tokenize $x$ with $\mathcal{T}$ to obtain input IDs and attention mask
    \STATE Run forward pass of $f$ to obtain token-level hidden states $X^l \in \mathbb{R}^{\ell \times d}$ at layer $l$
    \STATE Construct $L_1 \in \mathbb{R}^{\ell \times (\ell-1)}$ (upper-triangular), $L_2 \in \mathbb{R}^{\ell \times (\ell-1)}$ (lower-shifted identity)
    \STATE One-hot encode label $y \rightarrow Y \in \mathbb{R}^{c}$
    \STATE Compute mean token representation: $\bar{X}^l = \frac{1}{\ell} \sum_{j=1}^\ell X^l_j \in \mathbb{R}^{d}$

    \STATE $\texttt{meanXX} \mathrel{+}= \frac{1}{n} \bar{X}^l (\bar{X}^l)^\top$
    \STATE $\texttt{meanXY} \mathrel{+}= \frac{1}{n} \bar{X}^l Y^\top$
    \STATE $\texttt{cov0} \mathrel{+}= \frac{1}{n} (X^l)^\top L_1 L_1^\top X^l$
    \STATE $\texttt{cov1} \mathrel{+}= \frac{1}{n} (X^l)^\top L_1 L_2^\top X^l$
\ENDFOR

\STATE Compute top generalized eigenspace $U$ from $(\texttt{meanXY}\,\texttt{meanXY}^\top,\, \texttt{meanXX})$
\STATE Compute top generalized eigenspace $V$ from $(\texttt{cov1}\,\texttt{cov1}^\top,\, \texttt{cov0})$
\STATE Extract top-$k$ directions: $U_k \leftarrow$ first $k$ columns of $U$, $V_k \leftarrow$ first $k$ columns of $V$
\STATE Compute projection: $P_k \leftarrow V_k (V_k)^+$
\STATE Compute NTPS: $\texttt{NTPS} \leftarrow \|P_k U_k\|_F^2 / \|U_k\|_F^2$
\RETURN \texttt{NTPS}
\end{algorithmic}
\end{algorithm}
\end{minipage}
\vspace{-2mm}
\end{figure}

To illustrate that $U$ and $V$ capture fundamentally different axes of covariation, we extract the first embedding layer activations from a pretrained models (OpenELM-450M) on the Emotion\citep{emotion} dataset.  We then solve the two generalized eigenvalue problems in Theorem~\ref{thm:1} to obtain the projection matrices $U$ and $V$, each truncated to its top two eigenvectors.  
Figure~\ref{fig:uv-projections} visualizes two‑dimensional projections of representative words under the perception ($U^{\top}X$) and autoregressive ($V^{\top}X$) mappings. In the left panel, points colored by emotion form two nearly linearly separable clusters—positive versus negative—reflecting the label‑conditioned objective. The right panel shows that the same words in $V$‑space overlap heavily across emotion classes, indicating that next‑token prediction does not prioritize emotional polarity. Instead, the $V$‑space shows a clear grouping of adjectives vs. function words, suggesting that autoregressive training emphasizes syntactic category. In short, $U$ specializes in downstream, label‑relevant semantics such as sentiment, whereas $V$ encodes the structural, syntactic information essential for language modeling.

To obtain a continuous measure of overlap between the feature subspaces learned by $U$ (perception) and $V$ (autoregressive), we introduce the following alignment score.  Let
$P \;=\; V\,V^\dagger$
be the orthogonal projector onto the column space of $V$, where $V^\dagger$ is the Moore–Penrose pseudoinverse of $V$.  We then define
\[
\mathrm{NTPS}(U,V)\;=\;\frac{\|PU\|_{F}^{2}}{\|U\|_{F}^{2}}.
\]

Here, $\|P\,U\|_F^2$ is the total squared projection of $U$ onto $V$’s subspace, and $\|U\|_F^2$ is the total variance in $U$.  By construction, 
$0 \,\le\, \mathrm{NTPS}(U,V) \,\le\, 1,$ achieving $1$ if and only if the column spaces of $U$ and $V$ coincide, and approaching $0$ as they become orthogonal.  Higher values thus indicate greater alignment between the two objectives. The pseudocode for computing NTPS is provided in \Cref{alg:ntps}.

\begin{restatable}[Excess regression loss vs.\ subspace alignment]{theorem}{AlignmentLossTheorem}
\label{thm:alignment-loss}
Let \(U\) be the optimal perceptual encoder obtained from the generalized eigenproblem.
For any other encoder \(V\), define the orthogonal projector. Let \(\Delta\mathcal L := \mathcal L^{\!*}(V)-\mathcal L^{\!*}(U)\)
be the excess regression loss of \(V\).

Then there exist task-dependent positive constants
\(C_{\min}, C_{\max}\) such that
\begin{equation}
\boxed{\;
C_{\min}\bigl(1 - \mathrm{NTPS}(U,V)\bigr)
\;\le\;
\Delta\mathcal L
\;\le\;
C_{\max}\bigl(1 - \mathrm{NTPS}(U,V)\bigr).
\;}
\end{equation}
\end{restatable}
\textit{The proof of \Cref{thm:alignment-loss} is deferred to \Cref{app:thm2}}.

This shows that the extra regression loss of the autoregressive trained encoder $V$ over the perception trained $U$ is tightly controlled by their subspace alignment: as $NTPS(U,V)$ approaches 1 (perfect alignment), the excess loss vanishes, and as it decreases, the loss grows linearly within the constant bounds.

With these theoretical guarantees in hand, we now turn to empirical validation, demonstrating how our NTPS alignment score—derived in this simplified linear regime—predicts downstream performance across a range of nonlinear, pretrained models.

\begin{conclusionbox}
\textbf{Takeaway}: Our NTPS alignment score quantitatively captures how much of the perception‐trained subspace lies in the autoregressive subspace and is proved to bound the excess loss.
\end{conclusionbox}

\section{Empirical validation and practical utility of NTPS} 
\label{sec:alignment_performance_correlation}
In this section, we show that our NTPS—though derived under a linear model—captures meaningful alignment in nonlinear large‐scale LLMs. First, in \Cref{subsec:corr_linear} demonstrate a monotonic relationship between NTPS and downstream performance by showing Spearman correlations with both MSE loss and classification accuracy across eight pretrained models and 12 downstream perception datasets. Next in \Cref{subsec: NTPS_increase}, we examine how parameter-efficient LoRA finetuning shifts feature subspaces to boost NTPS, offering an interpretable lens on why LoRA works. Finally and most importantly, in \Cref{subsec: NTPS_lora}, we demonstrate that NTPS itself can predict the magnitude of accuracy gains from LoRA, making it a practical pre-screening metric for when fine-tuning will be most effective.

\subsection{Correlation between NTPS and linear probing performance}
\label{subsec:corr_linear}
First, we demonstrate that this score correlates with the downstream performance of eight pretrained models across 12 diverse datasets.

\begin{figure}[t]
    \centering
    \includegraphics[scale=1.0]{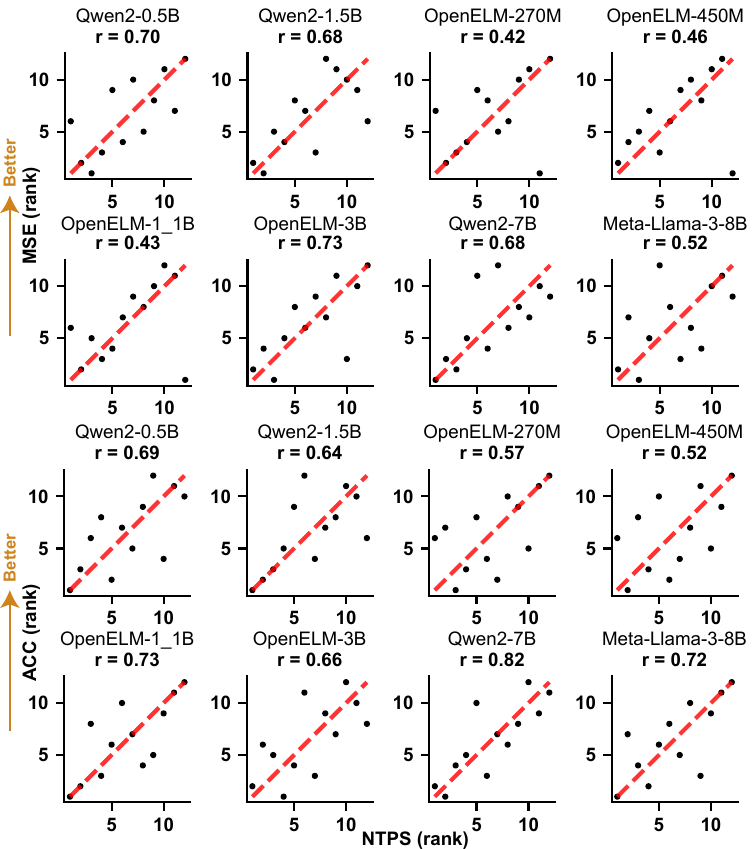}
    \caption{Correlation between NTPS and downstream MSE loss (rows 1 and 2), and between NTPS and accuracy (rows 3 and 4), with dashed lines indicating linear regression fits. \textbf{Higher alignment scores correspond to lower MSE loss and higher accuracy in downstream tasks.}}
    \label{fig:correlation}
\end{figure}

Similar to \cref{subsec:evidence}, we now include two more pretrained models (Qwen2-7B \citep{qwen2} and LLaMA3-8B \citep{llama3}), with the other models keep the same. Besides, we reuse the 12 datasets as described in \cref{subsec:evidence}. 


Downstream performance is measured in two complementary ways. First, we train a linear layer on each downstream dataset using ordinary least-sqaure (OLS) regression with close-form solution since our theoretical derivation in \cref{thm:1} is based on MSE loss (as shown in \cref{eq:autoregressive} and \cref{eq:perception}). And, we use the final MSE loss as the downstream performance metric. Second, we train a linear layer on each downstream dataset using logistic regression with saga optimizer under a CE loss since it better reflects practical usage. And, we use the final accuracy as the downstream performance metric. 

For each model across all datasets, we compute alignment scores over all layers (from the word‐embedding layer through the penultimate layer) and every $k$ proportion value from 0.05 to 0.95 in increments of 0.05. We then assess the monotonic relationship between alignment and downstream performance using Spearman’s \(r\in[-1,1]\). where \(r=1\) denotes perfectly concordant orderings.

To summarize each model succinctly, we report for each model the alignment score corresponding to the configuration that yields the strongest Spearman's $r$.
Besides, all results are obtained on training set to minimize confounding factors such as distribution shifts between train and test splits, which may obscure the true relationship between alignment score and downstream performance. 

The results, shown in \cref{fig:correlation}, reveal a clear trend: higher alignment scores are associated with lower MSE loss and higher accuracy in downstream tasks. This strong correspondence indicates that our alignment metric—despite its derivation under simplified linear assumptions—serves as an effective proxy for task alignment even in highly nonlinear models. It provides insight into when autoregressive training is beneficial for downstream tasks and can serve as a practical metric for anticipating the effectiveness of linear probing on pretrained models.

\begin{conclusionbox}
\textbf{Takeaway}: NTPS shows a clear monotonic relationship with downstream linear probe performance in LLMs—higher NTPS predicts better performance.
\end{conclusionbox}

\subsection{LoRA finetuning enhances NTPS}
\label{subsec: NTPS_increase}
As a step further, we provide an interpretation of why LoRA is effective for adapting pretrained LLMs to downstream tasks, through the lens of NTPS.

\begin{table}[t]
  \centering
  \caption{Relative improvement (\%) of NTPS after LoRA finetuning across models and downstream datasets. \textbf{NTPS is universally increased after LoRA finetuning, suggesting that LoRA finetuning enhances the overlap between feature subspaces of autoregressive training and downstream perception tasks.} For small models, NTPS slightly decreases probably because these models sacrifice the next-token prediction for higher downstream performance due to their limited capability.}
  \setlength{\tabcolsep}{1.8pt}
  \begin{tabular}{lcccccccc}
    \toprule
    Dataset & 
      \makecell{Qwen2\\0.5B} & 
      \makecell{Qwen2\\1.5B} & 
      \makecell{Qwen2\\7B} & 
      \makecell{OpenELM\\270M} & 
      \makecell{OpenELM\\450M} & 
      \makecell{OpenELM\\1.1B} & 
      \makecell{OpenELM\\3B} & 
      \makecell{LlaMA-3\\8B} \\
    \midrule
    Intent             &  1.9  &  1.6   &  53.2   &  -1.1  &  -1.0  &   0.9  &  88.5   &  77.8  \\
    Clickbait Title    &  1.7  &  1.9   &  70.1   &  -0.6  &  -0.3  &  -0.5  &  74.8   &  74.3  \\
    SST-2              &  1.4  &  4.3   & 105.9   &  -0.8  &  -0.7  &   2.1  & 100.0   & 102.5  \\
    Banking            &  2.3  &  0.9   &  67.8   &  -0.8  &  -0.6  &   0.7  &  79.5   &  87.1  \\
    Bias               &  3.1  &  3.6   &  78.3   &  -1.4  &  -1.4  &  -0.4  &  73.3   &  81.3  \\
    Emotion            &  2.3  &  2.9   & 109.0   &  -1.5  &  -2.3  &   0.5  &  83.4   &  80.0  \\
    SMS Spam           &  0.7  & -0.1   & 124.1   &  -1.4  &  -0.7  &   0.1  &  78.2   &  79.0  \\
    Medical            &  2.2  &  2.3   & 120.5   &  -0.8  &  -0.2  &   0.7  & 228.7   &  92.9  \\
    Rotten Tomatoes    &  0.8  &  0.2   & 105.6   &  -1.6  &  -0.4  &   0.1  &  88.6   &  84.3  \\
    Commonsense        &  1.1  &  0.7   & 108.8   &   1.1  &   0.6  &   1.0  &  91.6   &  95.0  \\
    Climate            &  1.5  &  1.8   & -14.5   &  -0.4  &  -1.8  &   1.2  &  76.7   &  99.9  \\
    IMDB               &  0.4  & -0.5   & 135.1   &  0.9   &  1.5   &   1.0  &  109.1  &  89.4  \\
    \bottomrule
  \end{tabular}
  \label{tab:ntps_improve}
\end{table}

Specifically, we compute NTPS for the same eight models and 12 datasets used in \cref{sec:alignment_performance_correlation}, using the exact same configuration for NTPS computation, but after applying LoRA. We adopt a consistent finetuning setup across all experiments: LoRA is applied to all QKV projection layers with rank 32, $\alpha = 32$, no bias, and a dropout rate of 0.05. For each input, we extract the mean token representation from the final transformer block and pass it into a linear classification head. We use the Adafactor optimizer with a learning rate of $10^{-4}$, $\epsilon$ values of $10^{-30}$ and $10^{-3}$, gradient clipping threshold of 1.0, decay rate of 0.8, and weight decay of $10^{-5}$. Training is conducted for 5000 steps with a cosine learning rate scheduler and a 5\% warm-up phase.

As shown in \cref{tab:ntps_improve}, NTPS increases in \textit{71 out of 96} runs after applying LoRA. This provides empirical support for our interpretation: LoRA may improve downstream task performance by adjusting the representations to better align the feature subspaces used for autoregressive pretraining and those required for downstream tasks, especially in large models. Interestingly, we do see that NTPS slightly decreases in small models like OpenELM-270M and OpenELM-450M, this is probably because these model sacrifice the next-token prediction capability in exchange for higher downstream performance due to its limited capability.

In summary, our empirical validation in \cref{subsec:corr_linear} and \cref{subsec: NTPS_increase} shows that NTPS not only correlates strongly with linear-probe performance across diverse pretrained models and downstream tasks, but also reliably increases after LoRA finetuning—confirming its utility as a practical measure of feature-subspace alignment. Building on these findings, we now turn to leveraging NTPS to predict the additional accuracy gains afforded by task-specific LoRA adaptations.

\begin{conclusionbox}
\textbf{Takeaway}: LoRA finetuning increases NTPS, shifting feature subspaces to better align autoregressive representations with downstream tasks.
\end{conclusionbox}

\subsection{Predicting LoRA finetuning gain with NTPS}
\label{subsec: NTPS_lora}
Now, we are going to show that our NTPS can also serve for practical usage, particularly for predicting the LoRA finetuning gain.
\begin{figure}[t]
    \centering
    \includegraphics[scale=1.0]{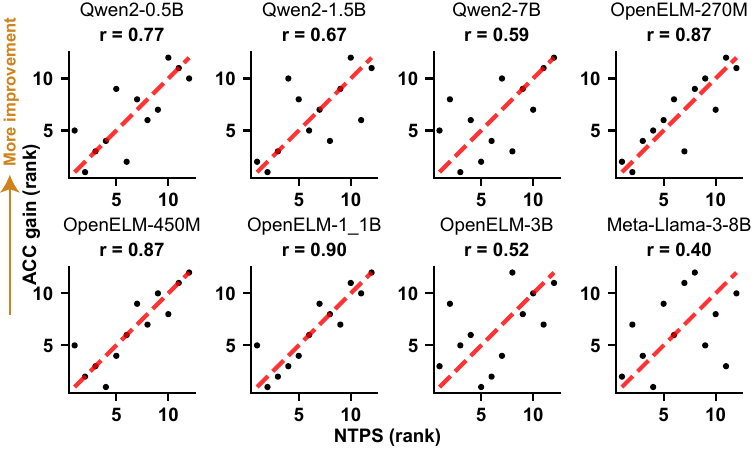}
    \caption{Correlation between NTPS and accuracy gain (LoRA finetuning accuracy-linear probing accuracy), with dashed lines indicating linear regression fits. \textbf{Higher alignment scores correspond to lower accuracy gain after LoRA finetuning in downstream tasks.}}
    \label{fig:correlation_lora}
\end{figure}

To evaluate whether NTPS can forecast the benefit of parameter-efficient fine-tuning, we measure the “ACC gain” as the difference between accuracy after LoRA adaptation and the baseline linear-probe accuracy (both on the test split). We reuse the same eight pretrained models (Qwen2-0.5B/1.5B/7B, OpenELM-270M/450M/1.1B/3B, and LLaMA3-8B) and the 12 downstream classification tasks described in Section \ref{sec:alignment_performance_correlation}. Again we (1) compute NTPS over all layers and $k$s exactly as before, (2) train a linear probe under CE to get baseline accuracy (AdamW optimizer, learning rate of $10^{-4}$; 50 epochs), and (3) apply LoRA (rank 32, $\alpha=32$, 5000 steps, Adafactor, 5\% warm-up) and record the adapted accuracy. Finally, we correlate NTPS with the observed LoRA gains using Spearman’s $r$.

As plotted in Figure \ref{fig:correlation_lora}, there is a clear monotonic relationship: models with lower NTPS enjoy larger accuracy gains from LoRA, whereas higher NTPS exhibit only modest improvements. Across the eight models, Spearman’s $r$ ranges from 0.40 up to 0.90 (higher absolute $r$ indicates stronger predictivity), confirming that NTPS is a reliable indicator of how much headroom remains for downstream adaptation. 

In practical terms, if a pretrained model yields a low NTPS on the target task, one can anticipate a sizable boost from LoRA; conversely, if a model yields a high NTPS, it is unlikely to benefit substantially from further finetuning. This makes NTPS a lightweight pre-screening tool to decide when parameter-efficient fine-tuning is most worthwhile.

\begin{conclusionbox}
\textbf{Takeaway}: NTPS inversely predicts the accuracy gains from LoRA finetuning: tasks with low initial alignment see the largest boosts.
\end{conclusionbox}

\section{Conclusion}
\label{sec:conclusion}
In conclusion, in this paper we have introduced the \emph{NTPS}, a simple yet powerful metric for measuring the alignment between the feature subspaces learned by autoregressive pretraining and those required for downstream perception tasks. In a linear setting, we proved that NTPS both upper- and lower-bounds the excess regression loss of an autoregressive encoder relative to an ideal perceptual encoder. Empirically, we demonstrate that NTPS—computed in closed form from pretrained representations and labeled data—correlates strongly classification accuracy across 12 diverse NLP datasets and eight pretrained models ranging from 270 M to 8 B parameters. By measuring NTPS before and after LoRA finetuning, we show that LoRA systematically increases subspace overlap and thus improves downstream performance as well as NTPS. Finally, we illustrate one important practical usage of our NTPS, i.e., predicting the accuracy gain after LoRA finetuning.

Our work still has several limitations that can be addressed in future research. First, NTPS is derived under a simplified linear setting and does not currently account for the role of attention mechanisms. Incorporating attention—possibly starting with linear attention models \citep{katharopoulos2020linear-attention}—could lead to a more precise formulation. Second, we have not yet explored how to select optimal configurations for computing NTPS in each model beforehand. For example, the choice of $k$ may depend on the model’s compression rate. Developing a more principled configuration strategy could improve efficiency and eliminate the need to exhaustively search over all possible settings.

\section{Broader impacts}
\label{sec:impacts}
This paper presents work whose goal is to advance the field of 
deep learning. There are many potential societal consequences 
of our work, none of which we feel must be specifically highlighted here.

\newpage
\appendix
\section{Appendix}
\subsection{Proof of Lemma 1}
\EquivMSECE*
\begin{proof}
For any \(j\neq y\), write
\[
  \langle w_y,\hat h\rangle - \langle w_j,\hat h\rangle
  = \bigl\langle w_y, h^* \bigr\rangle
  - \bigl\langle w_j, h^* \bigr\rangle
  + \bigl\langle w_y - w_j,\,\hat h - h^*\bigr\rangle.
\]
By definition the first term is at least \(\Delta\), and by Cauchy–Schwarz
\[
  \bigl|\langle w_y - w_j,\,\hat h - h^*\rangle\bigr|
  \;\le\;
  \|w_y - w_j\|_2\,\|\hat h - h^*\|_2
  \;\le\;
  M\,\|\hat h - h^*\|_2.
\]
Hence
\[
  \langle w_y,\hat h\rangle - \langle w_j,\hat h\rangle
  \;\ge\;
  \Delta - M\,\|\hat h - h^*\|_2.
\]
In particular, whenever
\(\|\hat h - h^*\|_2 < \Delta/M\) we have \(\langle w_y,\hat h\rangle>\langle w_j,\hat h\rangle\)
for all \(j\neq y\), so the arg‐max picks the correct token \(y\).  Since
\(\mathbb E\|\hat h - h^*\|_2^2 \to 0\) implies \(\|\hat h - h^*\|_2\to0\) in probability,
the probability of decoding error goes to zero.
\label{app:lemma1}
\end{proof}

\subsection{Proof of Theorem 1}
\ModelOptTheorem*
\begin{proof}
\label{app:thm1}

Consider a sentence, whose token representation is $X \in \mathbb{R}^{(d, \ell)}$, where $d$ is the hidden dimension size of each token and $\ell$ is the total number of tokens in this sentence. Consider two variants $X_1, X_2 \in \mathbb{R}^{d \times (\ell-1)}$ from $X$. For $X_1$, the $i$-th column represents the sum of the first $i$ tokens in $X$. For $X_2$, the $i$-th column represents the $(i+1)$-th token in $X$ so the whole matrix denotes the sequence of tokens from position 2 through $\ell$ in $X$.\\

Given a linear model to predict the $n$-th token given the sum of the previous $n-1$ tokens that contains a linear mapping $V \in \mathbb{R}^{d \times k}$ and a linear mapping $W \in \mathbb{R}^{k \times d}$. And set the loss $\mathcal{L}$ as:
\begin{equation}
\mathcal{L} = \mathbb{E}_{X_1, X_2}[\|W^\top V^\top X_1 - X_2\|_F^2],
\end{equation}

Denote $A = V^\top X_1 \in \mathbb{R}^{k\times(\ell-1)}$, then:
\begin{equation}
\begin{aligned}
    \mathcal{L} & = \mathbb{E}_{X_1, X_2}[\|W^\top A-X_2\|_F^2] \\
    & = \mathbb{E}_{X_1, X_2}[Tr((W^\top A-X_2)^\top (W^\top A-X_2))] \\
    & = \mathbb{E}_{X_1, X_2}[Tr(A^\top WW^\top A) - 2Tr(X_2^\top W^\top A)+Tr(X_2^\top X_2)]
\end{aligned}
\end{equation}

Taking the derivative of $\mathcal{L}$ w.r.t. $W$.
\begin{equation}
\begin{aligned}  
    \frac{\partial \mathcal{L}}{\partial W} & = \frac{\partial \mathbb{E}_{X_1, X_2}[Tr(A^\top WW^\top A) - 2Tr(X_2^\top W^\top A)+Tr(X_2^\top X_2)]}{\partial W} \\
    & = \frac{\partial \mathbb{E}_{X_1, X_2}[Tr(A^\top WW^\top A)]}{\partial W} - 2\frac{\partial \mathbb{E}_{X_1, X_2}[Tr(X_2^\top W^\top A)]}{\partial W} \\
    & = \frac{\partial \mathbb{E}_{X_1, X_2}[Tr(W^\top AA^\top W)]}{\partial W} - 2\frac{\partial \mathbb{E}_{X_1, X_2}[Tr(W^\top AX_2^\top )]}{\partial W} \\
    & = 2E_{X_1, X_2}[AA^\top W - AX_2^\top ] \\
    & = 2E_{X_1, X_2}[V^\top X_1X_1^\top VW - V^\top X_1X_2^\top ] \\ 
    & = 2E_{X_1, X_2}[V^\top X_1(X_1^\top VW - X_2^\top )]
\end{aligned}
\label{eq:l-wrt-w}
\end{equation}

To minimize $\mathcal{L}$, we set $\cref{eq:l-wrt-w} = 0$:
\begin{equation}
\begin{aligned}
    2E_{X_1, X_2}[V^\top X_1(X_1^\top VW - X_2^\top )] & = 0 \\
    E_{X_1, X_2}[V^\top X_1X_1^\top VW] & = E_{X_1, X_2}[V^\top X_1X_2^\top ]\\
\end{aligned}
\label{eq:l-wrt-w=0}
\end{equation}

Assume that $V^TX_1$ has rank $k$, then $V^TX_1X_1^TV$ is invertible, and we can express $W$ from \cref{eq:l-wrt-w=0} as below:
\begin{equation}
\begin{aligned}
    W & = E_{X_1, X_2}[(V^\top X_1X_1^\top V)^{-1}]E_{X_1, X_2}[V^\top X_1X_2^\top ] \\ 
    W & = (V^\top \mathbb{E}[X_1X_1^\top ]V)^{-1}V^\top \mathbb{E}[X_1X_2^\top ]
\end{aligned}
\end{equation}

By plugging in $W$ back to $\mathcal{L}$, we have:
\begin{equation}
\begin{aligned}
    \mathcal{L} & = \mathbb{E}[Tr(X_1^TV(V^\top \mathbb{E}[X_1X_1^\top ]V)^{-1}V^\top \mathbb{E}[X_1X_2^\top ]\mathbb{E}[X_2X_1^T]V(V^T\mathbb{E}[X_1X_1^T]V)^{-1}V^TX_1)] \\
    & \quad - 2\mathbb{E}[Tr(X_2^\top \mathbb{E}[X_2X_1^\top]V(V^\top \mathbb{E}[X_1X_1^\top] V)^{-1}V^\top X_1)] + \mathbb{E}[Tr(X_2^\top X_2)] \\
    & = Tr(\mathbb{E}[X_1^TV(V^\top \mathbb{E}[X_1X_1^\top ]V)^{-1}V^\top \mathbb{E}[X_1X_2^\top ]\mathbb{E}[X_2X_1^T]V(V^T\mathbb{E}[X_1X_1^T]V)^{-1}V^TX_1]) \\
    & \quad - 2Tr(\mathbb{E}[X_2^\top \mathbb{E}[X_2X_1^\top]V(V^\top \mathbb{E}[X_1X_1^\top] V)^{-1}V^\top X_1]) + Tr(\mathbb{E}[X_2^\top X_2]) \\
    & = Tr(\mathbb{E}[V^TX_1X_1^TV(V^\top \mathbb{E}[X_1X_1^\top ]V)^{-1}V^\top \mathbb{E}[X_1X_2^\top ]\mathbb{E}[X_2X_1^T]V(V^T\mathbb{E}[X_1X_1^T]V)^{-1}]) \\
    & \quad - 2Tr(\mathbb{E}[\mathbb{E}[X_2X_1^\top]V(V^\top \mathbb{E}[X_1X_1^\top] V)^{-1}V^\top X_1 X_2^\top]) + Tr(\mathbb{E}[X_2^\top X_2]) \\
    & = Tr((V^\top\mathbb{E}[X_1X_1^\top]V)(V^\top\mathbb{E}[X_1X_1^\top]V)^{-1}V^\top\mathbb{E}[X_1X_2^\top]\mathbb{E}[X_2X_1^\top]V(V^\top\mathbb{E}[X_1X_1^\top]V)^{-1}) \\
    & \quad - 2Tr(\mathbb{E}[X_2X_1^\top]V(V^\top\mathbb{E}[X_1X_1^\top]V)^{-1}V^\top\mathbb{E}[X_1X_2^\top]) + Tr(\mathbb{E}[X_2^\top X_2]) \\
    & = Tr(\mathbb{E}[X_2X_1^\top]V(V^\top\mathbb{E}[X_1X_1^\top]V)^{-1}V^\top\mathbb{E}[X_1X_2^\top]) \\
    & \quad - 2Tr(\mathbb{E}[X_2X_1^\top]V(V^\top\mathbb{E}[X_1X_1^\top]V)^{-1}V^\top\mathbb{E}[X_1X_2^\top]) + Tr(\mathbb{E}[X_2^\top X_2]) \\
    & = Tr(\mathbb{E}[X_2^\top X_2]) - Tr(\mathbb{E}[X_2X_1^\top]V(V^\top\mathbb{E}[X_1X_1^\top]V)^{-1}V^\top\mathbb{E}[X_1X_2^\top])
\end{aligned}
\end{equation}

Denote $R_{11} = \mathbb{E}[X_1X_1^\top ], R_{12} = \mathbb{E}[X_1X_2^\top ], R_{22} = \mathbb{E}[X_2X_2^\top ]$, minimizing $\mathcal{L}$ is solving the following maximization problem:
\begin{equation}
    \max_VTr(V^\top R_{12}R_{12}^\top V (V^\top R_{11}V)^{-1})
\end{equation}
which is equivalent to the following maximization problem:
\begin{equation}
    \max_{\tilde{V}: \tilde{V}^\top R_{11}\tilde{V} = I} Tr(\tilde{V}^\top R_{12}R_{12}^\top \tilde{V})
\label{eq:opt-v-tilde}
\end{equation}

And we can observe that the constraint is satisfied when:
\begin{equation}
    \tilde{V} = V(V^\top R_{11}V)^{-\frac{1}{2}}
\end{equation}
Thus, $\tilde{V}$ and $V$ share the same column space. And the subspace can be found via the optimization problem in \cref{eq:opt-v-tilde}, which yields to the generalized eigenvalue problem \cite{ghojogh2019eigenvalue}:
\begin{equation}
    R_{12}R_{12}^\top \tilde{V}=R_{11} \tilde{V}\Lambda
\label{eq:eigenval-autoregression}
\end{equation}
Since \cref{eq:opt-v-tilde} is a maximization problem, $\tilde{V}$ contains the eigenvectors of $(R_{12}R_{12}^\top, R_{11})$ that correspond to the top $k$ largest eigenvalues. And so $V$ spans the same column space as these eigenvectors.\\

From our definition we have $X_1=XL_1, X_2=XL_2$ with $L_1, L_2 \in \mathbb{R}^{(\ell, \ell-1)}$ defined as:
\begin{equation}
    L_1 = \begin{bmatrix}
        Q_{\ell-1} \\
        0
    \end{bmatrix},
\end{equation}

\begin{equation}
    L_2 = \begin{bmatrix}
        0\\
        I_{\ell-1}
    \end{bmatrix},
\end{equation}
where $Q_{l-1} \in \mathbb{R}^{l-1 \times l-1}$ is a unit upper triangular matrix (i.e. all entries on or above the diagonal are 1 and 0 below the diagonal). \\

Denote
\begin{equation}
  D \in \mathbb{R}^{\ell \times \ell}= L_1L_1^\top = \begin{bmatrix}
    Q_{\ell-1}Q_{\ell-1}^\top & 0 \\
    0 & 0
\end{bmatrix}
\end{equation}
\begin{equation}
  S \in \mathbb{R}^{\ell \times \ell} = L_1L_2^\top = \begin{bmatrix}
    0 & Q_{\ell-1} \\
    0 & 0
\end{bmatrix}
\end{equation}

Then we can rewrite the generalized eigenvalue problem in \cref{eq:eigenval-autoregression} as:
\begin{equation}
    \mathbb{E}[XSX^\top ]\mathbb{E}[XS^\top X^\top ]\tilde{V} = \mathbb{E}[XDX^\top ]\tilde{V}\Lambda
\end{equation}

Now let's consider a regression task with the label of the sentence $X$ denoted as $Y \in \mathbb{R}^c$.\\

Given a linear model to predict the label based on the sum of all tokens in the sentence that contains a linear mapping $U \in \mathbb{R}^{d \times k}$ and $Z \in \mathbb{R}^{k \times c}$. And set the learning objective as mean squared error (MSE) loss $\mathcal{L}^*$ as defined below:
\begin{equation}
    L^* = \mathbb{E} \left\| Z^\top U^\top X \, 1_{\mathrm{\ell \times 1}} - Y \right\|_F^2,
\end{equation}
where $1_{\mathrm{\ell \times 1}} \in \mathbb{R}^{\ell \times 1}$ is for summing the tokens in $X$.\\

Similarly, with the optimal $Z$, we will have the optimal $U$ sharing the same column space as $\tilde{U}$ that contains the eigenvectors corresponding to the largest k eigenvalues of the following generalized eigenvalue problem:
\begin{equation}
    \mathbb{E}[X1_{\mathrm{\ell \times 1}}Y^\top]\mathbb{E}[Y(X1_{\mathrm{\ell \times 1}})^\top]\tilde{U} = \mathbb{E}[(X1_{\mathrm{\ell \times 1}})(X1_{\mathrm{\ell \times 1}})^\top]\tilde{U}\Lambda,
\end{equation}
which simplifies to:
\begin{equation}
    \mathbb{E}[X1_{\mathrm{\ell \times 1}}Y^\top]\mathbb{E}[Y1_{\mathrm{1 \times \ell}}X^\top]\tilde{U} = \mathbb{E}[X1_{\mathrm{\ell \times \ell}}X^\top]\tilde{U}\Lambda,
\end{equation}
\end{proof}

\subsection{Proof of Theorem 2}
\AlignmentLossTheorem*
\begin{proof}
\label{app:thm2}
   Denote 
   $N:=\mathbb{E}[X1_{\mathrm{\ell \times \ell}}X^\top],\;M:=\mathbb{E}[X1_{\mathrm{\ell \times 1}}Y^\top].$

\begin{equation}
   \mathcal L^{\!*}(V)=\operatorname{Tr}(\mathbb E[YY^{\top}])\;-
   \operatorname{Tr}\!\bigl(V^{\top}MM^{\top}V\,(V^{\top}NV)^{-1}\bigr).
\end{equation}
\begin{equation}
   \mathcal L^{\!*}(U)=\operatorname{Tr}(\mathbb E[YY^{\top}])\;-
   \operatorname{Tr}\!\bigl(U^{\top}MM^{\top}U\,(U^{\top}NU)^{-1}\bigr).
\end{equation}

Recall that the columns of $U$ solve the generalized eigenvalue problem, and $U$ is $N$--orthonormal \big($U^{\top}NU=I_k$\big)
\begin{equation} \label{gep}
   MM^{\top}U\;=\;NU\Lambda, \qquad \Lambda=\operatorname{diag}(\Lambda_{11},\dots,\Lambda_{kk}),\;\Lambda_{11}\ge\dots\ge\Lambda_{kk}>0.
\end{equation}
the minimal regression loss is
\begin{equation}
   \mathcal L^{\!*}(U)=\operatorname{Tr}(\mathbb E[YY^{\top}])\;-
   \sum_{i=1}^{k}\Lambda_{ii}.
\end{equation}

Premultiplying \cref{gep} by $N^{-1/2}$ and defining the whitened basis $U^*:=N^{1/2}U$ gives the \emph{ordinary} symmetric eigenproblem. ($U$ and $U^*$ share the same subspace)
\begin{equation}\label{eq:whitened-eig}
   N^{-1/2}MM^{\top}N^{-1/2}\,U^*\;=\;U^*\,\Lambda.
\end{equation}

Introduce the whitened encoder $V^*:=N^{1/2}V\,(V^{\top}NV)^{-1/2}$ 
\begin{equation}
\begin{aligned}
   T(V):&=\operatorname{Tr}\!\bigl(V^{\top}MM^{\top}V\,(V^{\top}NV)^{-1}\bigr) \\
        &=\operatorname{Tr}\!\bigl(V^{*\top}\,N^{-1/2}MM^{\top}N^{-1/2}\,V^*\bigr) \\
        &=\operatorname{Tr}\bigl(V^{*\top}\,U^*\,\Lambda\,U^{*\top}\,V^*\bigr) \\
        &=\sum_{i=1}^{k}\Lambda_{ii}\,\| V^{*\top} u^*_i\|_2^{2}, \\
        &=\sum_{i=1}^{k}\Lambda_{ii}\, u^{*\top}_i V^* V^{*\top} u^*_i, \\
        &=\sum_{i=1}^{k}\Lambda_{ii}\, u^{\top}_i N^{1/2\top} V^* V^{*\top} N^{1/2}u_i, \\
        &=\sum_{i=1}^k \Lambda_{ii}\;u_i^{\top}\,N\,V\,(V^{\top}N\,V)^{-1}\,V^{\top}\,N\;u_i
\end{aligned}
\end{equation}

\begin{equation}
\begin{aligned}
\Delta\mathcal L
&:= \sum_{i=1}^{k}\Lambda_{ii}\,\bigl(1 - \lVert V^{*\top}u_i^*\rVert_2^2\bigr)\\
&=  \sum_{i=1}^{k}\Lambda_{ii}\,(1-u_i^{*\top}\,V^*V^{*\top}\,u_i^*) \\
&= \sum_{i=1}^{k}\Lambda_{ii}\,(1-u_i^{\top}\,N\,V\,(V^{\top}N\,V)^{-1}\,V^{\top}\,N\;u_i).
\end{aligned}
\end{equation}

Note that 
\begin{equation}
\begin{aligned}
      \text{NTPS}&=\frac{\|PU\|_F^2}{\|U\|_F^2} \\
      &= \frac{\operatorname{Tr}\!\bigl(U^{\top}P^{\top}PU\bigr)}{\|U\|_F^2} \\
      &= \frac{\operatorname{Tr}\!\bigl(U^{\top}PU\bigr)}{\|U\|_F^2} \\
      &= \frac{\sum_{i=1}^{k}u_{i}^{\top}Pu_{i}}{\|U\|_F^2} \\
      &= \frac{\sum_{i=1}^{k}u_{i}^{\top} V\,(V^\top V)^{-1} V^\top\,u_{i}}{\|U\|_F^2}.
\end{aligned}
\end{equation}

For each $i$, set
\[
  r_i = 1 - u_i^\top N V (V^\top N V)^{-1} V^\top N u_i.
\]
Writing $w = V a$ and minimizing
\[
  (u_i - V a)^\top N (u_i - V a)
    = u_i^\top N u_i - 2 a^\top V^\top N u_i + a^\top V^\top N V a
\]
over $a$ yields $a^* = (V^\top N V)^{-1}V^\top N u_i$, so
\[
  r_i = u_i^\top N u_i - u_i^\top N V (V^\top N V)^{-1} V^\top N u_i = \min_{w\in\mathrm{col}(V)}(u_i-w)^\top N(u_i-w).
\]

Since $x^\top N x\ge \lambda_{\min}(N)\|x\|^2$,
\begin{equation}
\begin{aligned}
  r_i &\ge \lambda_{\min}(N)\min_{w\in\mathrm{col}(V)}\|u_i-w\|^2, \\
        &= \lambda_{\min}(N)[\|u_i\|^2 - u_i^\top V(V^\top V)^{-1}V^\top u_i].
\end{aligned}
\end{equation}

Thus
\begin{equation}
\begin{aligned}
  \Delta\mathcal{L} &\ge \lambda_{\min}(N)\sum_i \Lambda_{ii}[\|u_i\|^2 - u_i^\top V(V^\top V)^{-1}V^\top u_i], \\
                   &\ge \lambda_{\min}(N)\Lambda_{\min}\|U\|_F^2(1-\mathrm{NTPS}), \\
\end{aligned}
\end{equation}

Similarly, $x^\top N x\le \lambda_{\max}(N)\|x\|^2$ gives
\begin{equation}
  r_i \le \lambda_{\max}(N)[\|u_i\|^2 - u_i^\top V(V^\top V)^{-1}V^\top u_i],
\end{equation}
and hence
\begin{equation}
  \Delta\mathcal{L} \le \lambda_{\max}(N)\Lambda_{\max}\|U\|_F^2(1-\mathrm{NTPS}).
\end{equation}

Combining,
\begin{equation}
  \boxed{\,
    \lambda_{\min}(N)\Lambda_{\min}\|U\|_F^2\,(1-\mathrm{NTPS})
    \le \Delta\mathcal{L}
    \le \lambda_{\max}(N)\Lambda_{\max}\|U\|_F^2\,(1-\mathrm{NTPS})
  }
\end{equation}
\end{proof}


\begin{thebibliography}{36}
\providecommand{\natexlab}[1]{#1}
\providecommand{\url}[1]{\texttt{#1}}
\expandafter\ifx\csname urlstyle\endcsname\relax
  \providecommand{\doi}[1]{doi: #1}\else
  \providecommand{\doi}{doi: \begingroup \urlstyle{rm}\Url}\fi

\bibitem[Saravia et~al.(2018)Saravia, Liu, Huang, Wu, and Chen]{emotion}
Elvis Saravia, Hsien-Chi~Toby Liu, Yen-Hao Huang, Junlin Wu, and Yi-Shin Chen.
\newblock {CARER}: Contextualized affect representations for emotion recognition.
\newblock In \emph{Proceedings of the 2018 Conference on Empirical Methods in Natural Language Processing}, pages 3687--3697, Brussels, Belgium, October-November 2018. Association for Computational Linguistics.
\newblock \doi{10.18653/v1/D18-1404}.
\newblock URL \url{https://www.aclweb.org/anthology/D18-1404}.

\bibitem[Radford et~al.(2018)Radford, Narasimhan, Salimans, Sutskever, et~al.]{gpt1}
Alec Radford, Karthik Narasimhan, Tim Salimans, Ilya Sutskever, et~al.
\newblock Improving language understanding by generative pre-training.
\newblock 2018.

\bibitem[Radford et~al.(2019)Radford, Wu, Child, Luan, Amodei, Sutskever, et~al.]{gpt2}
Alec Radford, Jeffrey Wu, Rewon Child, David Luan, Dario Amodei, Ilya Sutskever, et~al.
\newblock Language models are unsupervised multitask learners.
\newblock \emph{OpenAI blog}, 1\penalty0 (8):\penalty0 9, 2019.

\bibitem[Mann et~al.(2020)Mann, Ryder, Subbiah, Kaplan, Dhariwal, Neelakantan, Shyam, Sastry, Askell, Agarwal, et~al.]{gpt3}
Ben Mann, N~Ryder, M~Subbiah, J~Kaplan, P~Dhariwal, A~Neelakantan, P~Shyam, G~Sastry, A~Askell, S~Agarwal, et~al.
\newblock Language models are few-shot learners.
\newblock \emph{arXiv preprint arXiv:2005.14165}, 1:\penalty0 3, 2020.

\bibitem[Achiam et~al.(2023)Achiam, Adler, Agarwal, Ahmad, Akkaya, Aleman, Almeida, Altenschmidt, Altman, Anadkat, et~al.]{gpt4}
Josh Achiam, Steven Adler, Sandhini Agarwal, Lama Ahmad, Ilge Akkaya, Florencia~Leoni Aleman, Diogo Almeida, Janko Altenschmidt, Sam Altman, Shyamal Anadkat, et~al.
\newblock Gpt-4 technical report.
\newblock \emph{arXiv preprint arXiv:2303.08774}, 2023.

\bibitem[Touvron et~al.(2023{\natexlab{a}})Touvron, Lavril, Izacard, Martinet, Lachaux, Lacroix, Rozi{\`e}re, Goyal, Hambro, Azhar, et~al.]{llama1}
Hugo Touvron, Thibaut Lavril, Gautier Izacard, Xavier Martinet, Marie-Anne Lachaux, Timoth{\'e}e Lacroix, Baptiste Rozi{\`e}re, Naman Goyal, Eric Hambro, Faisal Azhar, et~al.
\newblock Llama: Open and efficient foundation language models.
\newblock \emph{arXiv preprint arXiv:2302.13971}, 2023{\natexlab{a}}.

\bibitem[Touvron et~al.(2023{\natexlab{b}})Touvron, Martin, Stone, Albert, Almahairi, Babaei, Bashlykov, Batra, Bhargava, Bhosale, et~al.]{llama2}
Hugo Touvron, Louis Martin, Kevin Stone, Peter Albert, Amjad Almahairi, Yasmine Babaei, Nikolay Bashlykov, Soumya Batra, Prajjwal Bhargava, Shruti Bhosale, et~al.
\newblock Llama 2: Open foundation and fine-tuned chat models.
\newblock \emph{arXiv preprint arXiv:2307.09288}, 2023{\natexlab{b}}.

\bibitem[Grattafiori et~al.(2024)Grattafiori, Dubey, Jauhri, Pandey, Kadian, Al-Dahle, Letman, Mathur, Schelten, Vaughan, et~al.]{llama3}
Aaron Grattafiori, Abhimanyu Dubey, Abhinav Jauhri, Abhinav Pandey, Abhishek Kadian, Ahmad Al-Dahle, Aiesha Letman, Akhil Mathur, Alan Schelten, Alex Vaughan, et~al.
\newblock The llama 3 herd of models.
\newblock \emph{arXiv preprint arXiv:2407.21783}, 2024.

\bibitem[Singh(2025)]{llama4}
Ajit Singh.
\newblock Meta llama 4: The future of multimodal ai.
\newblock \emph{Available at SSRN 5208228}, 2025.

\bibitem[Kumar et~al.(2022)Kumar, Raghunathan, Jones, Ma, and Liang]{kumar2022fine}
Ananya Kumar, Aditi Raghunathan, Robbie Jones, Tengyu Ma, and Percy Liang.
\newblock Fine-tuning can distort pretrained features and underperform out-of-distribution.
\newblock \emph{arXiv preprint arXiv:2202.10054}, 2022.

\bibitem[Liu et~al.(2023)Liu, Xie, Li, and Ma]{liu2023same}
Hong Liu, Sang~Michael Xie, Zhiyuan Li, and Tengyu Ma.
\newblock Same pre-training loss, better downstream: Implicit bias matters for language models.
\newblock In \emph{International Conference on Machine Learning}, pages 22188--22214. PMLR, 2023.

\bibitem[Vacareanu et~al.(2024)Vacareanu, Negru, Suciu, and Surdeanu]{vacareanu2024words}
Robert Vacareanu, Vlad-Andrei Negru, Vasile Suciu, and Mihai Surdeanu.
\newblock From words to numbers: Your large language model is secretly a capable regressor when given in-context examples.
\newblock \emph{arXiv preprint arXiv:2404.07544}, 2024.

\bibitem[Tenney et~al.(2019)Tenney, Xia, Chen, Wang, Poliak, McCoy, Kim, Van~Durme, Bowman, Das, et~al.]{tenney2019you}
Ian Tenney, Patrick Xia, Berlin Chen, Alex Wang, Adam Poliak, R~Thomas McCoy, Najoung Kim, Benjamin Van~Durme, Samuel~R Bowman, Dipanjan Das, et~al.
\newblock What do you learn from context? probing for sentence structure in contextualized word representations.
\newblock \emph{arXiv preprint arXiv:1905.06316}, 2019.

\bibitem[Jawahar et~al.(2019)Jawahar, Sagot, and Seddah]{jawahar2019does}
Ganesh Jawahar, Beno{\^\i}t Sagot, and Djam{\'e} Seddah.
\newblock What does bert learn about the structure of language?
\newblock In \emph{ACL 2019-57th Annual Meeting of the Association for Computational Linguistics}, 2019.

\bibitem[Tang et~al.(2024)Tang, Yang, and Song]{tang2024understandingLLM}
Ethan Tang, Boyu Yang, and Xinyi Song.
\newblock Understanding llm embeddings for regression.
\newblock In \emph{ICLR Workshop on Foundation Models}, 2024.

\bibitem[Gadre et~al.(2024)Gadre, Smyrnis, Shankar, Gururangan, Wortsman, Shao, Mercat, Fang, Li, Keh, et~al.]{gadre2024language}
Samir~Yitzhak Gadre, Georgios Smyrnis, Vaishaal Shankar, Suchin Gururangan, Mitchell Wortsman, Rulin Shao, Jean Mercat, Alex Fang, Jeffrey Li, Sedrick Keh, et~al.
\newblock Language models scale reliably with over-training and on downstream tasks.
\newblock \emph{arXiv preprint arXiv:2403.08540}, 2024.

\bibitem[Isik et~al.(2025)Isik, Ponomareva, Hazimeh, Paparas, Vassilvitskii, and Koyejo]{isik2025scaling}
Berivan Isik, Natalia Ponomareva, Hussein Hazimeh, Dimitris Paparas, Sergei Vassilvitskii, and Sanmi Koyejo.
\newblock Scaling laws for downstream task performance in machine translation.
\newblock In \emph{The Thirteenth International Conference on Learning Representations}, 2025.

\bibitem[Chen et~al.(2024)Chen, Huang, Gao, Wang, Yang, and Ji]{chen2024scaling}
Yangyi Chen, Binxuan Huang, Yifan Gao, Zhengyang Wang, Jingfeng Yang, and Heng Ji.
\newblock Scaling laws for predicting downstream performance in llms.
\newblock \emph{arXiv preprint arXiv:2410.08527}, 2024.

\bibitem[Balestriero and LeCun(2024)]{balestriero2024learning}
Randall Balestriero and Yann LeCun.
\newblock Learning by reconstruction produces uninformative features for perception.
\newblock \emph{arXiv preprint arXiv:2402.11337}, 2024.

\bibitem[Wu et~al.(2023)Wu, Lee, and Ge]{wu2023connecting}
Chenwei Wu, Holden Lee, and Rong Ge.
\newblock Connecting pre-trained language model and downstream task via properties of representation.
\newblock \emph{Advances in Neural Information Processing Systems}, 36:\penalty0 47216--47238, 2023.

\bibitem[Yang et~al.(2024)Yang, Yang, Hui, Zheng, Yu, Zhou, Li, Li, Liu, Huang, Dong, Wei, Lin, Tang, Wang, Yang, Tu, Zhang, Ma, Yang, Xu, Zhou, Bai, He, Lin, Dang, Lu, Chen, Yang, Li, Xue, Ni, Zhang, Wang, Peng, Men, Gao, Lin, Wang, Bai, Tan, Zhu, Li, Liu, Ge, Deng, Zhou, Ren, Zhang, Wei, Ren, Liu, Fan, Yao, Zhang, Wan, Chu, Liu, Cui, Zhang, Guo, and Fan]{qwen2}
An~Yang, Baosong Yang, Binyuan Hui, Bo~Zheng, Bowen Yu, Chang Zhou, Chengpeng Li, Chengyuan Li, Dayiheng Liu, Fei Huang, Guanting Dong, Haoran Wei, Huan Lin, Jialong Tang, Jialin Wang, Jian Yang, Jianhong Tu, Jianwei Zhang, Jianxin Ma, Jianxin Yang, Jin Xu, Jingren Zhou, Jinze Bai, Jinzheng He, Junyang Lin, Kai Dang, Keming Lu, Keqin Chen, Kexin Yang, Mei Li, Mingfeng Xue, Na~Ni, Pei Zhang, Peng Wang, Ru~Peng, Rui Men, Ruize Gao, Runji Lin, Shijie Wang, Shuai Bai, Sinan Tan, Tianhang Zhu, Tianhao Li, Tianyu Liu, Wenbin Ge, Xiaodong Deng, Xiaohuan Zhou, Xingzhang Ren, Xinyu Zhang, Xipin Wei, Xuancheng Ren, Xuejing Liu, Yang Fan, Yang Yao, Yichang Zhang, Yu~Wan, Yunfei Chu, Yuqiong Liu, Zeyu Cui, Zhenru Zhang, Zhifang Guo, and Zhihao Fan.
\newblock Qwen2 technical report, 2024.
\newblock URL \url{https://arxiv.org/abs/2407.10671}.

\bibitem[Mehta et~al.(2024)Mehta, Sekhavat, Cao, Horton, Jin, Sun, Mirzadeh, Najibi, Belenko, Zatloukal, and Rastegari]{openelm}
Sachin Mehta, Mohammad~Hossein Sekhavat, Qingqing Cao, Maxwell Horton, Yanzi Jin, Chenfan Sun, Iman Mirzadeh, Mahyar Najibi, Dmitry Belenko, Peter Zatloukal, and Mohammad Rastegari.
\newblock {OpenELM}: {An} {Efficient} {Language} {Model} {Family} with {Open} {Training} and {Inference} {Framework}.
\newblock \emph{arXiv.org}, April 2024.
\newblock URL \url{https://arxiv.org/abs/2404.14619v1}.

\bibitem[Bhuvaneshwari()]{intent_classification}
Bhuvaneshwari.
\newblock Intent classification.
\newblock \url{https://huggingface.co/datasets/Bhuvaneshwari/intent\_classification}.
\newblock Accessed: 2025-04-29.

\bibitem[Chakraborty et~al.(2016)Chakraborty, Paranjape, Kakarla, and Ganguly]{clickbait}
Abhijnan Chakraborty, Bhargavi Paranjape, Sourya Kakarla, and Niloy Ganguly.
\newblock Stop clickbait: Detecting and preventing clickbaits in online news media.
\newblock In \emph{Advances in Social Networks Analysis and Mining (ASONAM), 2016 IEEE/ACM International Conference on}, pages 9--16. IEEE, 2016.

\bibitem[Socher et~al.(2013)Socher, Perelygin, Wu, Chuang, Manning, Ng, and Potts]{sst2}
Richard Socher, Alex Perelygin, Jean Wu, Jason Chuang, Christopher~D. Manning, Andrew Ng, and Christopher Potts.
\newblock Recursive deep models for semantic compositionality over a sentiment treebank.
\newblock In \emph{Proceedings of the 2013 Conference on Empirical Methods in Natural Language Processing}, pages 1631--1642, Seattle, Washington, USA, October 2013. Association for Computational Linguistics.
\newblock URL \url{https://www.aclweb.org/anthology/D13-1170}.

\bibitem[Patel()]{bias_identification}
Priya Patel.
\newblock Bias identification.
\newblock \url{https://huggingface.co/datasets/PriyaPatel/Bias\_identification}.
\newblock Accessed: 2025-04-29.

\bibitem[Casanueva et~al.(2020)Casanueva, Temcinas, Gerz, Henderson, and Vulic]{bank77}
I{\~{n}}igo Casanueva, Tadas Temcinas, Daniela Gerz, Matthew Henderson, and Ivan Vulic.
\newblock Efficient intent detection with dual sentence encoders.
\newblock In \emph{Proceedings of the 2nd Workshop on NLP for ConvAI - ACL 2020}, mar 2020.
\newblock URL \url{https://arxiv.org/abs/2003.04807}.
\newblock Data available at https://github.com/PolyAI-LDN/task-specific-datasets.

\bibitem[Almeida et~al.(2011)Almeida, Hidalgo, and Yamakami]{sms_spam}
Tiago~A. Almeida, Jose Maria~Gomez Hidalgo, and Akebo Yamakami.
\newblock Contributions to the study of sms spam filtering: New collection and results.
\newblock In \emph{Proceedings of the 2011 ACM Symposium on Document Engineering (DOCENG'11)}, 2011.

\bibitem[McCreery et~al.(2020)McCreery, Katariya, Kannan, Chablani, and Amatriain]{medical}
Clara~H. McCreery, Namit Katariya, Anitha Kannan, Manish Chablani, and Xavier Amatriain.
\newblock Effective transfer learning for identifying similar questions: Matching user questions to covid-19 faqs, 2020.

\bibitem[Pang and Lee(2005)]{rotten_tomatoes}
Bo~Pang and Lillian Lee.
\newblock Seeing stars: Exploiting class relationships for sentiment categorization with respect to rating scales.
\newblock In \emph{Proceedings of the ACL}, 2005.

\bibitem[Talmor et~al.(2019)Talmor, Herzig, Lourie, and Berant]{commonsenseqa}
Alon Talmor, Jonathan Herzig, Nicholas Lourie, and Jonathan Berant.
\newblock {C}ommonsense{QA}: A question answering challenge targeting commonsense knowledge.
\newblock In \emph{Proceedings of the 2019 Conference of the North {A}merican Chapter of the Association for Computational Linguistics: Human Language Technologies, Volume 1 (Long and Short Papers)}, pages 4149--4158, Minneapolis, Minnesota, June 2019. Association for Computational Linguistics.
\newblock \doi{10.18653/v1/N19-1421}.
\newblock URL \url{https://aclanthology.org/N19-1421}.

\bibitem[Bingler et~al.(2023)Bingler, Kraus, Leippold, and Webersinke]{climate_sentiment}
Julia Bingler, Mathias Kraus, Markus Leippold, and Nicolas Webersinke.
\newblock How cheap talk in climate disclosures relates to climate initiatives, corporate emissions, and reputation risk.
\newblock Working paper, Available at SSRN 3998435, 2023.

\bibitem[Maas et~al.(2011)Maas, Daly, Pham, Huang, Ng, and Potts]{imdb}
Andrew~L. Maas, Raymond~E. Daly, Peter~T. Pham, Dan Huang, Andrew~Y. Ng, and Christopher Potts.
\newblock Learning word vectors for sentiment analysis.
\newblock In \emph{Proceedings of the 49th Annual Meeting of the Association for Computational Linguistics: Human Language Technologies}, pages 142--150, Portland, Oregon, USA, June 2011. Association for Computational Linguistics.
\newblock URL \url{http://www.aclweb.org/anthology/P11-1015}.

\bibitem[Hui and Belkin(2020)]{hui2020evaluation}
Like Hui and Mikhail Belkin.
\newblock Evaluation of neural architectures trained with square loss vs cross-entropy in classification tasks.
\newblock \emph{arXiv preprint arXiv:2006.07322}, 2020.

\bibitem[Katharopoulos et~al.(2020)Katharopoulos, Vyas, Pappas, and Fleuret]{katharopoulos2020linear-attention}
Angelos Katharopoulos, Apoorv Vyas, Nikolaos Pappas, and Fran{\c{c}}ois Fleuret.
\newblock Transformers are rnns: Fast autoregressive transformers with linear attention.
\newblock In \emph{International conference on machine learning}, pages 5156--5165. PMLR, 2020.

\bibitem[Ghojogh et~al.(2019)Ghojogh, Karray, and Crowley]{ghojogh2019eigenvalue}
Benyamin Ghojogh, Fakhri Karray, and Mark Crowley.
\newblock Eigenvalue and generalized eigenvalue problems: Tutorial.
\newblock \emph{arXiv preprint arXiv:1903.11240}, 2019.

\end{thebibliography}
\end{document}